\def\year{2021}\relax
\documentclass[letterpaper]{article} 
\usepackage{aaai21}  
\usepackage{times}  
\usepackage{helvet} 
\usepackage{courier}  
\usepackage[hyphens]{url}  
\usepackage{graphicx} 
\usepackage{natbib}  
\usepackage{caption} 
\nocopyright 
\usepackage{graphicx}
\usepackage{amsmath}
\usepackage{xcolor}

\usepackage{algorithm}
\usepackage{algorithmic}
\usepackage{multirow}
\usepackage{wrapfig}
\usepackage[switch]{lineno}
\usepackage{amsthm}
\newtheorem{theorem}{Theorem}[section]

\newtheorem{lemma}[theorem]{Lemma}

\newcommand{\secref}[1]{Section \ref{#1}}

\renewcommand{\eqref}[1]{Eq.~(\ref{#1})}

\newcommand{\teacher}{f_{\theta}}
\newcommand{\student}{f_{\phi}}
\newcommand{\dteacher}{\frac{\partial f_{\theta}}{\partial \theta}}

\title{Teacher-Student Consistency \\ For Multi-Source Domain Adaptation}
\author{
    Ohad Amosy,\textsuperscript{\rm 1} 
    Gal Chechik\textsuperscript{\rm 1}\textsuperscript{\rm 2}
}
\affiliations{

    \textsuperscript{\rm 1}Bar-Ilan University
    \textsuperscript{\rm 2}Nvidia Research
    \\

}

\begin{document}
\maketitle

\begin{abstract}

In \textit{Multi-Source Domain Adaptation} (MSDA), models are trained on samples from multiple source domains and used for inference on a different, target, domain. 
Mainstream domain adaptation approaches learn a joint representation of source and target domains. Unfortunately, a joint representation may emphasize features that are useful for the source domains but hurt inference on target (\textit{negative transfer}), or remove essential information about the target domain (\textit{knowledge fading}).

We propose \textbf{Mu}lti-source \textbf{S}tudent \textbf{T}eacher (MUST), a novel procedure designed to alleviate these issues. The key idea has two steps: First, we train a teacher network on source labels and infer pseudo labels on the target. Then, we train a student network using the pseudo labels and regularized the teacher to fit the student predictions. This regularization helps the teacher predictions on the target data remain consistent between epochs. 
Evaluations of MUST on three MSDA benchmarks: digits, text sentiment analysis, and visual-object recognition show that MUST outperforms current SoTA, sometimes by a very large margin. We further analyze the solutions and the dynamics of the optimization showing that the learned models follow the target distribution density, implicitly using it as information within the unlabeled target data. 
\end{abstract}

\section{Introduction}
Multi-source domain adaptation is a fundamental problem in AI. In many scenarios, labeled samples are collected from multiple distinct \textit{source} distributions, and we wish to use them to learn a model that is then applied to another, \textit{target}, distribution. 
In many cases, unlabeled target samples are available during training.
For example, images may be taken under several known weather or lighting conditions, medical data may be collected using different versions of a sensor and product reviews may be collected for  different products. In all these cases, we wish to learn from all source domains and infer on unlabeled data from a new distribution.  

\begin{figure}[!h]
    \begin{center}
        \includegraphics[width=1\linewidth]{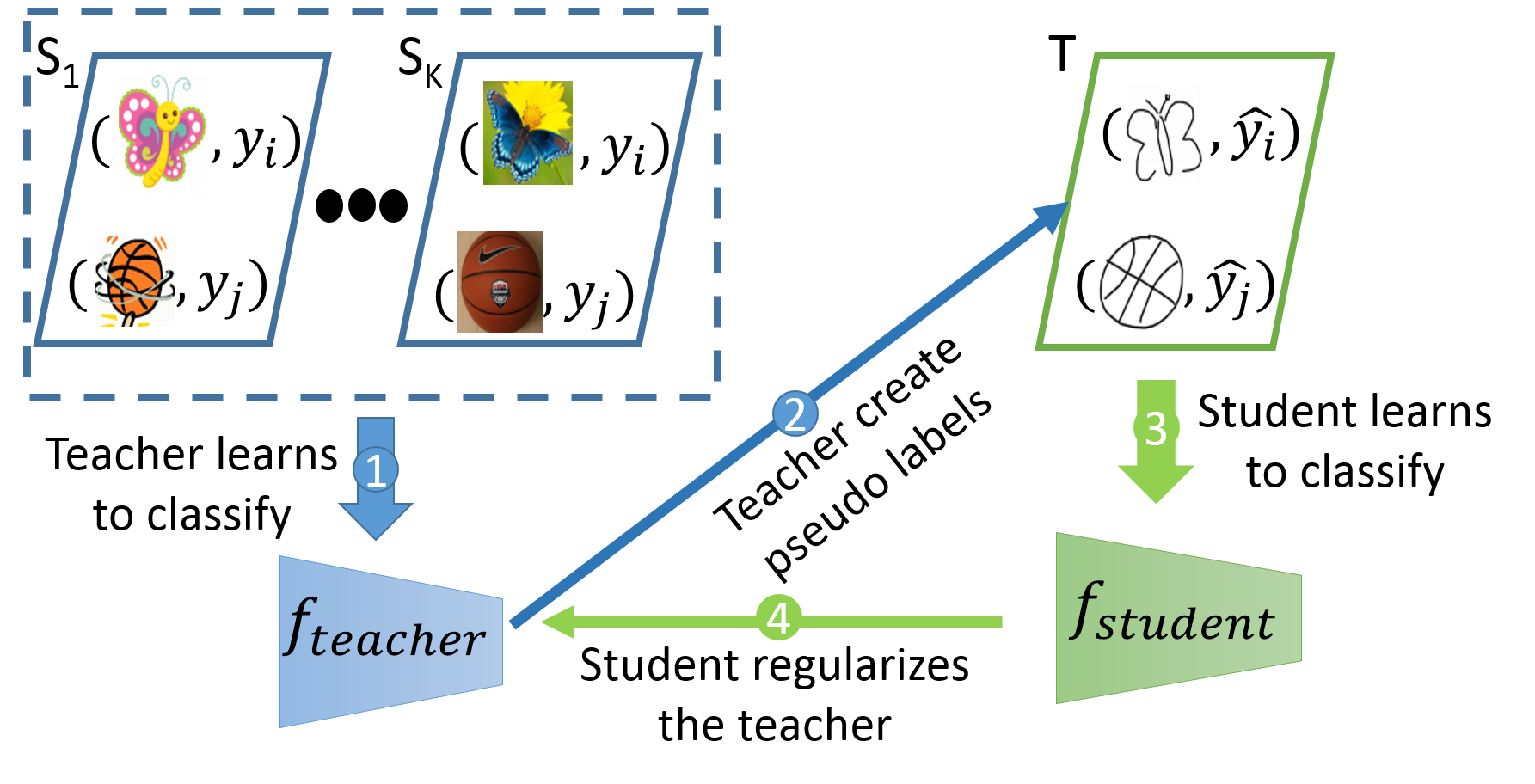}
    \end{center}
    \caption{\textbf{MUST Training}. (1) The source domains are used to train a teacher network. (2) The teacher predicts the labels of the target domain. (3) The target domain with the teacher predictions are used to train a student network. (4) The student regularizes the teacher by ensuring its consistency with the student on the target.
    This example demonstrate the separation between the source domains and the student, in a way that prevents the student from learning features that do not exist in the target. For example, when source images contains color and the target does not, the teacher may use color feature for classification, but the student will not be exposed to those features.
   }
    \label{fig:figure1}
\end{figure}

For single-source unsupervised DA, mainstream approaches learn a joint representation of the source and the target domains. This can be achieved by aligning the source and target distributions to minimize their divergence \cite{wilson2018survey, kouw2019review}, by normalizing source and target domain data \cite{li2018adaptive-adaBN,cariucci2017autodial}, or by mapping between the source and target domains with GANs \cite{wilson2018survey}. Once this representation is learned, one can train a model using source domain labels and use it to in the target domain.

Unfortunately, learning a joint representation of source and target domains has two significant drawbacks, which become severe when learning from multiple sources. First, \textit{negative transfer}: The joint representation may emphasize features that are useful for the source domains but hurt inference on target data. Second, \textit{knowledge fading}: without labels available for the target, learning removes essential information about the target domain. This problem is more severe in the multi-source domains, because aligning more domains may remove more target-relevant features.

We propose to address these two issues by training two separate networks, a \textit{teacher network} trained on all source domains and a \textit{student network} trained on the target domain. 
Because there are no target domain labels, the student uses pseudo labels, produced by the teacher.
This way, the student does not learn source features and we prevent negative transfer. In addition, the student learns from the raw target data and we prevent knowledge fading.

The training procedure is as follows: First, a teacher network is trained on the source domains, and infers soft pseudo labels (posteriors) for target domain samples. Next, a student network is trained on target samples using the teacher predictions as labels.  This way, the information from the source domains is projected to the target domain. 

Training on the target domain using pseudo labels from the teacher is a natural way to train the student network. Unfortunately, those pseudo labels tend to change rapidly from one epoch to the next, leading to poor performance of the student network. This instability occurs because  the teacher network us trained on the source domains and is blind to the prediction it makes on the target domain.
Small changes in the teacher model may cause large changes of the pseudo labels on the target, because the teacher model is tested on samples that are outside of its training distribution. To address this issue,  
%
we constrain the teacher network to make predictions that are similar to the student predictions. This operates as a regularizer, that reduces prediction instability across epochs. Figure 1 illustrates this learning setup, and the architecture which we name \textbf{MUST}, for \textit{MUlti-Source Student-Teacher} training .
We analyze this process theoretically and empirically  section \secref{sec:MUST}.


Our novel contributions are as follows: (1) We describe a new training procedure for multi-source UDA. (2) We analyze the dynamics, showing empirically and analytically that it tends to converge to the low-density areas of the target distribution. (3) We evaluate MUST using three MSDA classification benchmarks:  Digits (MNIST-like), sentiment analysis (Amazon data) and visual object recognition (DomainNet). MUST reaches new SOTA on all three benchmarks. The reduction in error compared with an estimate of the upper bound obtained with labeled target data is dramatic, reducing average error rate by 76\% on the digits dataset. 


\section{Related work}
\label{sec:related_work}
Single-source domain adaptation (SSDA) has been extensively studied. See \cite{wilson2018survey,kouw2019review} for recent literature surveys.

\textbf{Multi-source DA:}
Compared to the vast amount of research done on single-source UDA, multi-source DA (MSDA) is still less explored. The theory of MSDA was studied by \citet{ben2010theory}. They suggested using divergence between source and target domains to find a theoretical bound for generalization error. \citet{mansour2009domain} introduces a new divergence measure and \citet{mansour2009domain-linear} suggested presenting the target hypotheses as a weighted combination of source hypotheses. Inspired by DANN \cite{ganin2016domain}, an SSDA approach, \citet{zhao2018adversarial-MSDA} proposed multi-source domain adversarial networks (MDAN), an adversarial loss to find a representation indistinguishable between all source domains and the target domain and, at the same time, informative enough for the given task. \citet{li2018adaptive-adaBN} uses different batch norm per domain (adaBN), letting the model learn different batch statistics in different domains. Motivated by \cite{mansour2009domain-linear}, Deep Cocktail Networks (DCTN) \citet{xu2018deep-cocktail} used adversarial learning to minimize the discrepancy between the target and each of the multiple source domains. Multi-Domain Matching Network (MDMN) \cite{li2018extracting-MDMN} increases domain similarities not only between the source and target domains but also within the source domain themselves based on a Wasserstein-like measure. Moment matching (M3SDA) \cite{peng2019moment} minimizes the first order moment-related distance between all source and target domains. Domain aggregation network (DARN) \cite{wen2019domain-agg} learns to weight the source domains to find the optimal balance between increasing the effective sample size and excluding irrelevant data.

\textbf{Reverse validation:}
Unlike supervised learning, for DA tasks, hyperparameter tuning can not be done using cross validation, since the target data is from a different distribution and has no labels. 

\citet{zhong2010cross-ping-pong} proposed using reverse validation (RV) for hyperparameter tuning with no target labeled data. RV is estimated by first splitting the source (labeled) and target (unlabeled) data into training and validation sets. The training set is used to train a classifier with any UDA method, and infer pseudo-labels over the target validation set. Those pseudo-labels (and some labeled target data if exist) are used then to train another classifier with the same UDA method, but with the pseudo-labeled target data as the source and the source data, without the labels, as the target. This reverse classifier is evaluated on the source validation data and its loss is the RV. The parameters that gain the lowest RV are selected.

\textbf{Student-teacher approaches:}
Several authors proposed using a teacher-student architecture for DA. 

\cite{bousmalis2017unsupervised}, \cite{meng2018adversarial-ts-speech} and \cite{meng2019domain-ts-speech} perform DA for speech recognition tasks. \citet{bousmalis2017unsupervised} uses pairs of samples from source and target domains, which are frame-by-frame synchronized. The teacher is a trained model on the source domain, that remains constant. In the training process, the teacher and student predict the labels to a pair of related source and target samples. The student minimizes the KL divergence between the outputs. \citet{meng2018adversarial-ts-speech} uses the same idea as \cite{bousmalis2017unsupervised}, but add to the student an adversarial objective that encourages it to learn condition-invariant features. \citet{meng2019domain-ts-speech} also followed \cite{bousmalis2017unsupervised}, but replaces the teacher and student models with an attention based encoder-decoder.

\citet{ruder2017knowledge-st-da} used multiple teachers, one per source domain. The teachers are trained in a supervised manner and remain constant through the adaptation process. The student model is then trained to imitate a weighted sum of the teacher’s predictions.

In \cite{french2017self} the student is trained on the source domain, while the weights of the teacher network are an exponential moving average of those of the student. Target data is passed through both teacher and student with different augmentations and penalizes the mean squared difference between their predictions. 

MUST differs from these approaches in several components: (1) MUST does not assume source and target synchronization. (2) MUST trains the teacher and the student jointly. (3) In MUST the teacher and the student are different and independent networks.

\section{Our approach}
\label{sec:MUST}
Our approach is designed to reduce negative transfer while maintaining the original representation of the target domain. 
The key idea is to train a model only on raw target samples. This way, the model is not exposed directly to features that are not target related. In addition, the model trains on the original target samples, so no information fading accrue.
We use two networks to learn separately from the source domains and target domain. We train a "\textit{Teacher}" network on labeled data from the source domains. The teacher transfer its knowledge to a "\textit{Student}" network by training the student network on the target data, using the teacher predictions on that data as labels. 

However, teacher network predictions on the target are noisy throughout the training process. This noise interferes with the student convergence. To solve this issue, we regularize the teacher to fit also to the student predictions on the target. Since the student is trained on the teacher previous prediction, this regularization helps the teacher produce more consistent predictions on the target.

We now formally define the learning setup and our approach. 
Let $\{S^{k}\}_{i=1}^{K}$ be a collection of source domains, where the $k^{th}$ domain has $N_k$ labeled samples $\{(x_i^k,y_i^k)\}_{i=1}^{N_k}$. Similarly, $T$ is a target domain, with $N_T$ unlabeled samples $\{z_i\}_{i=1}^{N_T}$. 
The multi-source DA problem aims to find a hypothesis $f$, which minimizes the test target error $\epsilon(f)=E_{z\sim T}[p_T(y|z)-p(y|z,f)]$, where $p_T(y|z)$ denotes the conditional distribution of the target domain and $p(y|z,f)$ denotes the conditional distribution of the predicted label.

Our algorithm trains two networks. The "\textit{Teacher}" network trained on labeled source samples and the "\textit{Student}" network trained on target samples with the teacher predictions. These networks are trained jointly, using two losses:
\textbf{(1) Teacher learns to classify}. Train a teacher classifier $f_{\theta}$ using samples from the source domains, by minimizing

\begin{equation}
\label{eq:l1}
    L_{teacher-clf}=\sum_{k=1}^{K} \frac{1}{N_k}\sum_{i=1}^{N_k} l_1(f_{\theta}(x_i^k),y_i^k).
\end{equation}

\textbf{(2) The teacher trains the student}: Use $f_{\theta}$ to give soft labels to the target domain and use them to train the student $f_{\phi}$. For that we minimize
\begin{equation}
\label{eq:l2}
    L_{student}= \frac{1}{N_T}\sum_{i=1}^{N_T} l_2(\student(z_i),\teacher(z_i)).
\end{equation}
$L_{student}$ is a function of both the student network $f_{\phi}$ and the teacher network $f_{\theta}$. We used its derivatives w.r.t. the student parameter $\phi$ to train the student network. 

\textbf{Combining the losses.}
Finally, the student loss is treated as a regularizer, and the loss for training the teacher network becomes
\begin{equation}
\label{eq:l3}
    L_{teacher} = L_{teacher-clf}+\lambda \cdot L_{student},
\end{equation}
where $\lambda$ is a hyper parameter selected using reverse validation. $l_1$ and $l_2$ are loss functions. The training process is summarized in algorithm \ref{alg:training}.

\begin{algorithm}[!h]
  \caption{MUST training procedure}
  \label{alg:training}  
    \begin{algorithmic}[1]
      \STATE {\bfseries Input:} Source domains samples $\{\{(x_i^k,y_i^k)\}_{i=1}^{N_k}\}_{k=1}^{K}$, target domain samples $\{z_i\}_{i=1}^{N_T}$ and a hyperparameter $\lambda$
      
      \FOR{$t=1..steps$}
      \STATE $src = random(K)$   //Choose random source
      \STATE $X_{src},Y_{src}$=sample-batch$(\{(x_i^{src},y_i^{src}\})_{i=1}^{N_{src}})$
      \STATE Calculate $L_{teacher-clf}$ using \eqref{eq:l1}
      \STATE $X_{tgt} =$ sample-batch$(\{z_i\}_{i=1}^{N_T}$)
      \STATE Calculate $L_{student}$ using \eqref{eq:l2}
      \STATE Update $\phi$ to minimize $L_{student}$
      \STATE Calculate $L_{teacher}$ using \eqref{eq:l3}
      \STATE Update $\theta$ to minimize $L_{teacher}$ 
      \ENDFOR
    \end{algorithmic}
\end{algorithm}

We now show that the gradients for the teacher model became small in one of two cases: when the student agree with the teacher on the target or when moving the decision boundaries away from the target samples.
Let $f_{\theta}(x)$ be the teacher network and $f_{\phi}(x)$ be the student. When $l_1$ is the cross entropy loss and $l_2$ is the Euclidean $L_2$, the teacher loss can be written as 
\begin{multline}
    L_{teacher}=E_{(x,y)\sim S}[y log\teacher(x)+(1-y) log(1-\teacher(x))] \\
    + \lambda E_{z\sim T}[l_2(\teacher(z),\student(z))].
\end{multline}
%
The teacher gradient decent update rule is: 
\begin{equation}
    \theta ^{t+1}=\theta ^{t}+\eta \frac{\partial L_{teacher}}{\partial \theta},
\end{equation}
where $\eta$ is the learning rate. 

\begin{lemma}
\label{lemma:margin}
Let the teacher network $\teacher$ be a binary classifier parametrized by $\theta$, whose last layer is a sigmoid $\teacher=\sigma(g_\theta(x))$. Let  $\student$ be the student network. If $\forall{}z \in T: |\frac{\partial g_\theta(z)}{\partial \theta}|\leq A$ and $|g_\theta(z)| \geq \rho \geq 0$ than using $L_2$ for the student loss bounds the gradient:
$\frac{\partial L_{teacher}}{\partial \theta} \leq$ \\ 
$$E_{(x,y)\sim S}[y \frac{\dteacher(x)}{\teacher(x)}] + 2\lambda E_{z\sim T}[(\teacher(z)-\student(z))]\frac{A}{e^{\rho}}. $$
\end{lemma}

\begin{proof}

For any $\theta$, the loss derivative is:
\begin{multline}
\label{eq:teacer-loss-step}
    \frac{\partial L_{teacher}}{\partial \theta} = 
    E_{(x,y)\sim S}[y \frac{\dteacher(x)}{\teacher(x)}+(1-y) \frac{-\dteacher(x)}{1-\teacher(x)}] + \\
    \lambda E_{z\sim T}[2(\teacher(z)-\student(z))\dteacher(z)].
\end{multline}
Now, using $\forall{}z \in T: |\frac{\partial g_\theta(z)}{\partial \theta}|\leq A$ and $|g_\theta(z)| \geq \rho \geq 0$ gives:
\begin{equation}
\label{eq:claim-A}
    \frac{\partial f_\theta(z)}{\partial \theta} = 
    \frac{\frac{\partial g_\theta(z)}{\partial \theta}}{2+e^{-g_\theta(z)}+e^{g_\theta(z)}}
    \leq \frac{\frac{\partial g_\theta(z)}{\partial \theta}}{e^{g_\theta(z)}}
    \leq \frac{A}{e^{\rho}}.
\end{equation}
Plugging (\ref{eq:claim-A}) into (\ref{eq:teacer-loss-step}) gives us 
$\frac{\partial L_{teacher}}{\partial \theta} \leq E_{(x,y)\sim S}[y \frac{\dteacher(x)}{\teacher(x)}$  
$$+(1-y) \frac{-\dteacher(x)}{1-\teacher(x)}] + 2\lambda E_{z\sim T}[(\teacher(z)-\student(z))\frac{A}{e^{\rho}}]. $$
Since $\frac{A}{e^{\rho}}$ is a constant it follows that: 

\begin{multline}
    \frac{\partial L_{teacher}}{\partial \theta} \leq E_{(x,y)\sim S}[y \frac{\dteacher(x)}{\teacher(x)}+(1-y) \frac{-\dteacher(x)}{1-\teacher(x)}] \\
    + 2\lambda E_{z\sim T}[(\teacher(z)-\student(z))]\frac{A}{e^{\rho}}.
\end{multline}
\end{proof}
This result gives an upper bound for the changes in $\theta$ between epochs.
The teacher converges when $\frac{\partial L_{teacher}}{\partial \theta}$ vanishes. If the teacher is in a local minimum with respect to the source data $E_{(x,y)\sim S}[y \frac{\dteacher(x)}{\teacher(x)}+(1-y) \frac{-\dteacher(x)}{1-\teacher(x)}]$ is 0. $E_{z\sim T}[(\teacher(z)-\student(z))]\frac{A}{e^{\rho}}$ decreases to zero by fitting the student and the teacher perfectly on the target data, or by increasing $\rho$. $|g_\theta(z)|$ is the distance from the decision boundary, so increasing $\rho$ is achieve by moving the decision boundaries away from the target samples. We demonstrated this phenomenon in section \ref{sec:experiments}.

\section{Experimental evaluation}
\label{sec:experiments}
We evaluate MUST using three real-world datasets. First, in the task of sentiment analysis we used the Amazon product reviews benchmark \cite{blitzer2006domain-amazon-dataset} and in the task of image classification we used the digit recognition datasets and DomainNet dataset \cite{peng2019moment}. Code and trained models will be made available at https://github.com/amosy3/MUST.

\textbf{Implementation details:}
We use cross entropy as the loss function of Eq. (\ref{eq:l1}). In (\ref{eq:l2}), loss gradients are calculated on both the estimated probability distribution and the true distribution, which led to a non-stable training when using a cross entropy loss. To overcome this we used $L_1$ loss in (\ref{eq:l2}).

We followed \cite{li2018adaptive-adaBN} and use different batch norm layers for each domain. This way, the model can effectively use the data from all domains using one model. We also followed \cite{french2017self} and use confidence thresholding, to stabilize the training process. For each target sample, if the teacher softmax layer maximum is lower than the confidence threshold (denoted as $C_{th}$), the student will not use that sample for training.

\subsection{Sentiment analysis}
\label{sec:sentiment}
Sentiment analysis aims to decide if the sentiment of a given text is positive or negative. In the DA setup, source and target domains are text from different distribution. For the Amazon reviews dataset different domains are different products.

\textbf{The data:} 
The Amazon reviews dataset contains 27677 reviews on four kinds of products: Books, DVDs, Electronics and Kitchen appliances. Each review is labeled as positive or negative. For fair comparison, we followed the experimental setup of \cite{zhao2018adversarial-MSDA} with the code provided by the authors. Reviews are encoded as 5000 dimensional feature vectors of unigrams and bigrams, with binary labels indicating sentiment. 
We conduct four experiments: we pick one product as the target domain and the rest as source domains for each of them. Each source domain has 2000 labeled examples and the target test set has 3000 to 6000 examples. 

\textbf{Experimental setup:}
As in \cite{zhao2018adversarial-MSDA}, we used the same basic network structure with one 5000 units input layer and three hidden layers with 1000, 500 and 100 units. As MUST uses different batch norm per domain, we added a batch norm layer to the input layer. We used SGD optimizer with learning rate of 0.001, a moment of 0.9 for training and experimented with several hyperparameters configurations ($\lambda=0.25,0.5,1.0$ and $C_{th}=0.6,0.9$), choosing between them using reverse validation.

\textbf{Compared approaches:}
We compare our results with the state-of-the-art results reported in \cite{ruder2017knowledge-st-da} and \cite{wen2019domain-agg}: 
\textbf{(1) mSDA}: \citet{chen2012marginalized} uses stacked denoising autoencoder to learn new higher-level representations.
\textbf{(2) DANN}: \citet{ganin2016domain} uses adversarial loss to create a representation that a domain classifier is unable to classify from which domain the feature representation originated. DANN is a single to single DA method and can not be directly applied in a multiple source domains setting. For the multi-source setup \citet{wen2019domain-agg} merged all the source domains and use it as one large source domain.
\textbf{(3) Knowledge Adaptation (KA)}: \citet{ruder2017knowledge-st-da} uses multiple teachers, one for each domain and another general teacher that trained on all the sources combined, to train a student to imitate the weighted sum of the teacher’s predictions. 
\textbf{(4) Moment matching (M3SDA)}: \citet{peng2019moment} minimizes the first order moment-related distance between all source and target domains.
\textbf{(5) adversarial MSDA (MDAN)}: \citet{zhao2018adversarial-MSDA} optimizes adversarial loss to find a representation that is indistinguishable between the k source domains and the target domain, but informative enough for our desired task to succeed.
\textbf{(6) Multi-Domain Matching Network (MDMN)}:  \citet{li2018extracting-MDMN} increases domain similarities not only between the source and target domains but also within the source
domain themselves based on Wasserstein-like measure.
\textbf{(7) Domain aggregation network (DARN)}: \citet{wen2019domain-agg} learns to weight the source domains to find the optimal balance between increase effective sample size and excluding irrelevant data.
\textbf{(8) Target}: a baseline model that trained on the target labels. 

\begin{table}
    \centering
    \setlength{\tabcolsep}{3pt}
    \scalebox{0.90}{
    \begin{tabular}{ |l|c|c|c|c||c| } 
    \quad Method & Books & DVD & Elect & Kitn & Avg.\\ 
    \hline
    mSDA (Chen 2012) & 77.0 & 78.6 & 82.0 & 84.3 & 80.5\\ 
    DANN (Ganin 2016) & 79.1 & 80.6 & 85.3 & 85.6 & 82.6\\
    KA (Ruder 2017)  & 80.1 & 80.9 & 83.1 & 86.5 & 82.5\\
    $M^3SDA$ (Peng 2019) & 79.4 & 80.8 & 85.5 & 86.5 & 83.1\\
    MDAN (Zhao 2018) & 80.0 & 81.7 & 84.8 & 86.8 & 83.3\\ 
    MDMN (Li 2019) & 80.1 & 81.6 & 85.6 & 87.1 & 83.6\\
    DARN (Wen 2019) & 79.9 & 81.6 & 85.8 & 87.2 & 83.6\\
    MUST (ours) & \textbf{80.5} & \textbf{81.9} & \textbf{86.3} & \textbf{87.9} & \textbf{84.2} \\
    \hline
    \hline
    Labeled target & 84.2 & 83.8 & 86.4 & 88.7 & 85.8\\
    Error reduction (vs Tgt) & 14\% & 14\%  & 83\% & 46\% & 27\%\\
    \hline
    \end{tabular}
    }
    \caption{Sentiment classification accuracy}
    \label{tab:amazon}
\end{table}

\textbf{Results:}
The accuracy of the various methods is summarized in Table \ref{tab:amazon}. Clearly, MUST outperforms all other methods. Note that as \cite{wilson2018survey} points out, some works use target labels for hyperparameter tuning, which can be interpreted as an upper bound on how well the method could perform. We stress that unlike these studies, we used reverse validation for parameter tuning, which is much more similar to real life scenario. This makes our result even more significant.

\subsection{Image classification}
In the task of image classification we used the digit recognition datasets and DomainNet dataset \cite{peng2019moment}.

\textbf{Experimental setup:}
For the image classification experiments, we trained a ResNet-152 with SGD optimizer with learning rate of 0.001 and a moment of 0.9. We tried a few hyperparameters configuration ($\lambda=0.25,0.5$ and $C_{th}=0.25,0.5,0.95$) and used the average accuracy of the student on the source domain as a metric to choose between them.   

\subsection*{Digit recognition}
\label{sec:digits-recognition}
\textbf{The data:}
The digits dataset is a union of several datasets: (1) MNIST: low resolution black and white images of hand written digits. (2) MNIST-M: consists of MNIST digits blended with random color patches. (3) SVHN: contains low resolution images of digits from google street view home number.  (4) SynthDigits: synthetic SVHN-like dataset digits. Those datasets are union to one dataset were each dataset is consider as a different domain. 

\textbf{Compared approaches:}
We compare our results with the current state-of-the-art results as reported in \cite{wen2019domain-agg}.

\begin{table}[!h]
    \centering
    \setlength{\tabcolsep}{3pt}
    \scalebox{0.90}{
    \begin{tabular}{ |l|c|c|c|c||c| } 
    \quad Method & MNIST & MNISTM & SVHN & SYN. & Avg.\\ 
    \hline
    DANN (Ganin 2016) & 96.4 & 60.1 & 70.2 & 83.8 & 77.6\\
    $M^3SDA$ (Peng 2019) & 97.0 & 65.0 & 71.7 & 80.1 & 78.4\\
    MDAN (Zhao 2018) & 97.1 & 64.1 & 77.7 & 85.5 & 81.1\\ 
    MDMN (Li 2019) & 97.2 & 64.3 & 76.4 & 85.8 & 80.9\\
    DARN (Wen 2019) & 98.1 & 67.1 & 81.6 & 86.8 & 83.4\\
    MUST (ours) & \textbf{98.9} & \textbf{83.8} & \textbf{86.0} & \textbf{96.1} & \textbf{91.2} \\
    \hline
    \hline
    Labeled target & 99.0 & 94.7 & 87.6 & 97.0 & 94.6\\
     Error reduction (vs Tgt) & 77\% & 60\%  & 73\% & 91\% & 76\%\\
    \hline
    \end{tabular}
    }
    \caption{Digit recognition accuracy}
    \label{tab:digits}
\end{table}

\textbf{Results:}
As summarized in Table \ref{tab:digits}, MUST outperforms current state-of-the-art MSDA methods by significant margin. The largest improvement is the adaptation to MNIST-M dataset and SYNTH dataset. This make sense due to the fact that MUST focus on common features between the sources and the target. MNIST-M is based on MNIST and SYNTH is based on SVHN, so those datasets have more common features than typical datasets, which leads to the big improvements.

\subsection*{DomainNet}
\label{sec:domainnet}
We next evaluate MUST in a problem of adaptation for visual object recognition.  

\textbf{The data:}
DomainNet \cite{peng2019moment} is a recent challenging dataset designed to evaluate multi-source domain adaptation. It is by far the largest UDA dataset, containing six aligned domains (clipart, infograph, painting, quickdraw, Real and Sketch) and about 6 million images distributed among 345 categories. This dataset is far more challenging than previous digit-based datasets. As shown in \cite{peng2019moment}, even state-of-the-art methods fail to adapt well between its domains.

\textbf{Compared approaches:}
We compared our results to a set of methods as reported by \cite{peng2019moment} in Table \ref{tab:domainnet}.
\cite{peng2019moment} use SSDA methods as a baseline to MSDA in two setups. First, \textit{single best}, where the model is trained on each of the source domains and evaluate on the target domain. The best test score is chosen. Second, a setup called \textit{source combine}, where all samples from the source domains are aggregated as if coming from a single source domain. The model is trained on the aggregated dataset and evaluated on the target domain. 

We compared MUST to several SSDA approaches:
\textbf{(0) No adaptation} Using only source domain samples. 
\textbf{(1) DAN} \cite{long2015learning-DAN} applied MMD to layers embedded in a reproducing kernel Hilbert space, matching higher order statistics of the two distributions.
\textbf{(2) RTN} \cite{long2016unsupervised-RTN} uses residual layers to bridge over components that do not transfer well between domains. 
\textbf{(3) JAN} \cite{long2017deep-JAN} aligns the joint distributions of multiple layers across domains based on a joint maximum mean discrepancy.
\textbf{(4) DANN} \cite{ganin2016domain} uses adversarial loss to create a representation that a domain classifier is unable to classify from which domain the sample originated.
\textbf{(5) ADDA} \cite{tzeng2017adversarial-ADDA} combines discriminative modeling and a GAN loss.
\textbf{(6) SE} \cite{french2017self} ensemble over time between student and teacher models. 
\textbf{(7) MCD} \cite{saito2018maximum-MCD} finding two classifiers that maximize the discrepancy on the target sample and then generate features that minimize this discrepancy.

In addition to SSDA baselines, there are MSDA baselines: 
\textbf{(8) DCTN} \cite{xu2018deep-cocktail} used adversarial learning to minimize the discrepancy between the target and each source domain.
\textbf{(9) $M^3SDA-\beta$} \cite{peng2019moment} minimizes the first order moment-related distance between all source and target domains.

\textbf{Results:}
As summarized in Table \ref{tab:domainnet}. MUST outperforms current state-of-the-art MSDA methods in 4 out of 6 tasks. 
As was shown by \cite{peng2019moment}, MSDA methods are consistently better than SSDA baselines. The only exception was Quickdraw. No adaptation baselines score 11.8\% and 13.3\% while some single-source DA baselines improve this score up to 16.2\%. Surprisingly the MSDA methods achieve lower score than no adaptation (7.2\% for DCTN and 6.3\% for $M^3SDA$. This indicates that a negative transfer occurred when applying MSDA methods. MUST design to reduce negative transfer and compare to other MSDA methods we almost doubled the accuracy in the Quickdraw task up to 12.2\%, getting much closer to the no adaptation baselines. This is still blow some SSDA baseline. That demonstrates the difficulty of avoiding negative transfer in the multi-source setting and shows that there is room for further improvement.

\begin{table}[h]
    \small
    \centering
    \setlength{\tabcolsep}{3pt}
    \scalebox{0.85}{
    \begin{tabular}{|l|l|c|c|c|c|c|c||c|}
        & Models & {Clip} & {Info} & {Paint} & {Quick} & {Real} & {Skt} & {Avg} \\\hline
    \multirow{8}{*}{\begin{tabular}[c]{@{}c@{}}Single \\ Best \end{tabular}}    
        & Source Only & 39.6 & 8.2  & 33.9 & 11.8 & 41.6 & 23.1 & 26.4 \\
        & DAN (Long 2015) & 39.1 & 11.4 & 33.3 & \textbf{16.2} & 42.1 & 29.7 & 28.6 \\
        & RTN (Long 2016) & 35.3 & 10.7 & 31.7 & 13.1 & 46.0 & 26.5 & 26.3 \\
        & JAN (Long 2017) & 35.3 & 9.1  & 32.5 & 14.3 & 43.1 & 25.7 & 26.7 \\
        & DANN \tiny{(Ganin 2016)}& 37.9 & 11.4 & 33.9 & 13.7 & 41.5 & 28.6 & 27.8 \\
        & ADDA \tiny{(Tzeng 2017)}& 39.5 & 14.5 & 29.1 & 14.9 & 41.9 & 37.0 & 28.4 \\
        & SE (French 2017) & 31.7 & 12.9 & 19.9 & 7.7  & 33.4 & 26.3 & 22.0 \\
        & MCD (Saito 2018) & 42.6 & 19.6 & 42.6 & 3.8  & 50.5 & 33.8 & 32.2 \\ 
        \hline
    \multirow{8}{*}{\begin{tabular}[c]{@{}c@{}}Source \\ Combine\end{tabular}} 
        & Source Only & 47.6 & 13.0 & 38.1 & 13.3 & 51.9 & 33.7 & 32.9 \\
        & DAN (Long 2015) & 45.4 & 12.8 & 36.2 & 15.3 & 48.6 & 34.0 & 32.1 \\
        & RTN (Long 2016) & 44.2 & 12.6 & 35.3 & 14.6 & 48.4 & 31.7 & 31.1 \\
        & JAN (Long 2017) & 49.0 & 11.1 & 35.4 & 12.1 & 45.8 & 32.3 & 29.6 \\
        & DANN (Ganin 2016) & 45.5 & 13.1 & 37.0 & 13.2 & 48.9 & 31.8 & 32.6 \\
        & ADDA (Tzeng 2017) & 47.5 & 11.4 & 36.7 & 14.7 & 49.1 & 33.5 & 32.2 \\
        & SE (French 2017) & 24.7 & 3.9  & 12.7 & 7.1  & 22.8 & 9.1  & 16.1 \\
        & MCD (Saito 2018) & 54.3 & 22.1 & 45.7 & 7.6  & 58.4 & 43.5 & 38.5 \\ \hline
    \multirow{3}{*}{\begin{tabular}[c]{@{}c@{}}Multi- \\ Source\end{tabular}} 
        & DCTN (Xu 2018) & 48.6 & 23.5 & 48.8 & 7.2  & 53.5 & 47.3 & 38.2 \\
        & M3SDA (Peng 2019)& 58.6 & \textbf{26.0} & \textbf{52.3} & 6.3 & 62.7 & 49.5 & 42.5 \\
        & MUST (ours) & \textbf{60.8} & 20.5 & 48.2 & 12.2 & \textbf{65.1} & \textbf{49.8} & \textbf{42.8}\\ 
        \hline
        \hline
        & Target  & 71.6 & 36.7  & 68.6 & 69.6  & 81.8 & 65.8  & 65.7 \\
        \hline
    \end{tabular}
    }
    \caption{Image classification accuracy on DomainNet}
    \label{tab:domainnet}
\end{table} 

\section{Analysis}
\subsection{Analysis of optimization dynamics}
A closer look into the training process can help us understand the dynamic interplay of the two networks. Figure \ref{fig:learning_graphs} traces the target accuracy of the student and teacher together with $L_{teacher-clf}$ and $L_{student}$ during training. In addition, we plot the number of samples that cross the confidence threshold and the reverse validation score.
Interestingly, learning follows through four phases.  
\newline\textbf{(1) Teacher learns}: In the first 1000 iterations, the teacher learns to perform the classification task on the source domains, its accuracy grows, but no prediction on the target domain is above the confidence threshold, so the student does not get any label data to train on.
\newline\textbf{(2) Sync}: Between iterations 1000 and 2000, the teacher confidence in the target domain grows and the student starts getting its first labels to train on. Surprisingly, there is no change in the student target accuracy, even though the teacher gives the student target soft labels to train on, and the accuracy of the teacher is around its saturation value. This phenomenon can be understood by looking at the losses. 
The teacher gives the student soft labels to learn from, but they are inconsistent between iterations, so the student can not fit them well. The student regularization forces the teacher to predict more consistent predictions. This extra regularization causes the teacher accuracy to decrease a bit.
\newline\textbf{(3) Student learns}: Around iterations 2000-2500, the teacher and the student start to improve together. The student accuracy improves fast on the target data and also fits better with the teacher predictions. The reverse validation loss decline indicates that the student focuses on features that are relevant to both the source and target data.
\newline\textbf{(4) Saturation}: After iteration 2500, both student and teacher reach saturation. The student achieves higher target accuracy than the teacher, which indicates that it reduces the negative transfer effect.

\begin{figure}[!h]
    \includegraphics[width=0.5\textwidth]{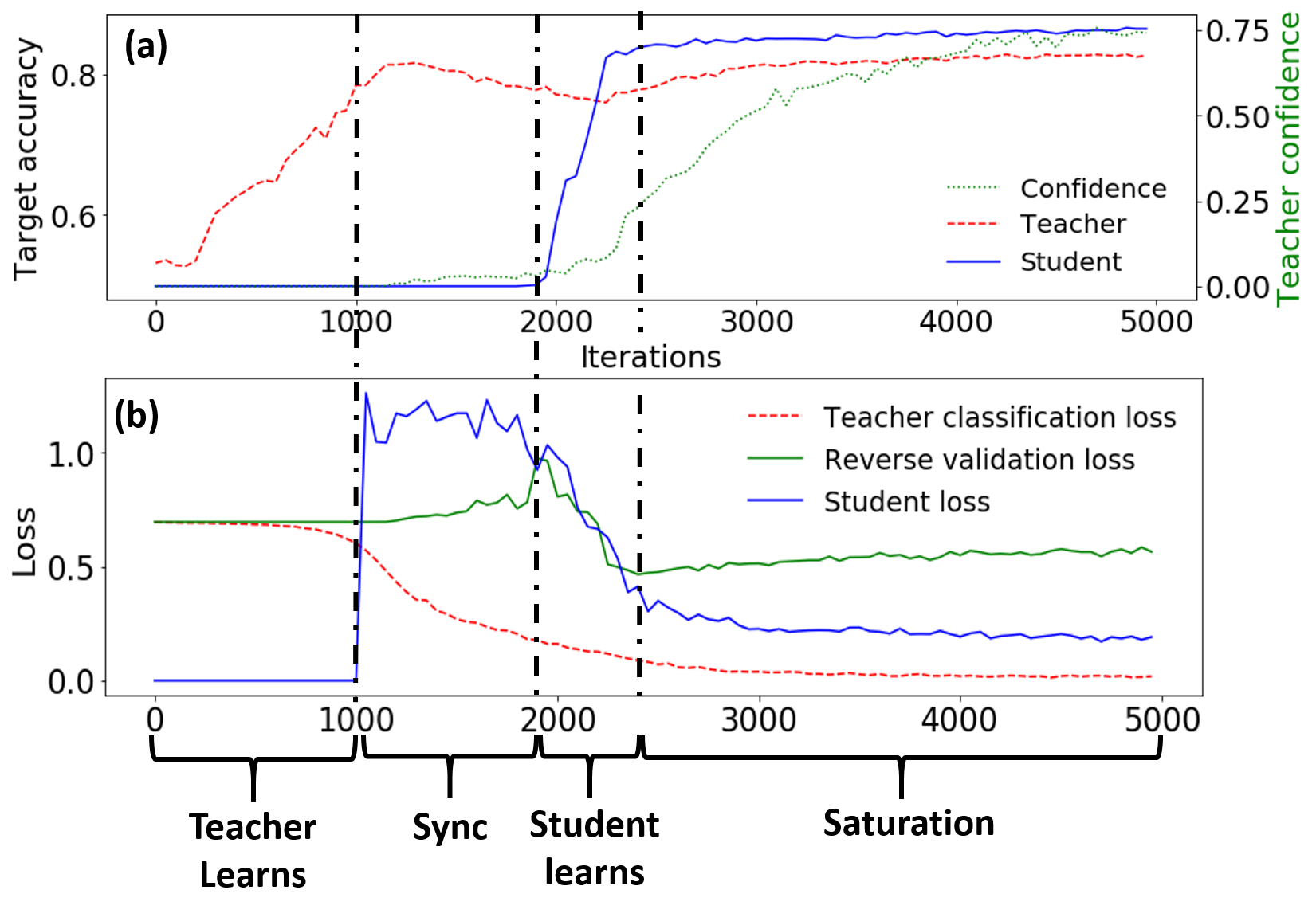}
    \caption{\textbf{The dynamics of teacher and student joint-learning}. (a) Target accuracy of teacher (red) and student (blue) during training and the percent of samples that passed the confidence threshold (green). (b) Loss functions values during training $L_{teacher-clf}$ (red), $L_{student}$ (blue) and reverse validation loss (green).} 
    \label{fig:learning_graphs}
\end{figure}

\subsection{Qualitative analysis}
To demonstrate the effectiveness of MUST regularization, we conducted the following experiment: We trained a teacher network on source data with and without the student loss regularization. Then, we tracked the teacher prediction for each target sample along the training process. We calculated the standard deviation per sample using a sliding-window over training steps, and used the average of all the samples to measure the model consistency. Figure \ref{fig:noise_reduction} shows that for a non-regularized teacher, the average standard deviation rises a few times during training which implies noisy predictions on the target. MUST regularization reduces average standard deviation along training and ensures predictions consistency.

\begin{figure}[!h]
    \begin{center}
        \includegraphics[width=0.85\linewidth]{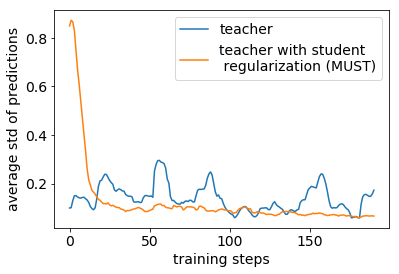}
    \end{center}
    \caption{\textbf{Consistency of teacher predictions on target}. Shown is the standard deviation of predictions during training, with (orange) and without (blue) student regularization.}
    \label{fig:noise_reduction}
\end{figure}
To get a better understanding of the classifier properties learned using MUST, we used an illustrative example. We generated 2D data from 2 classes and trained a 2-layer fully-connected neural network to perform binary classification. We train and visualized the decision boundary of 20 different initializations of the source only and MUST models.

Figure \ref{fig:avoid-clusters} shows the MUST teacher moves its decision boundary away from the target samples and avoid cutting the target clusters as lemma \ref{lemma:margin} predicted. This is an important feature because if the target data is clustered, we expect the decision boundary to cross between the clusters (regions with low density) and not cut in the middle of a cluster (high density regions). This property is known as the \textit{clustering assumption} and it is widely use in semi-supervised methods.

\begin{figure}[!h]
    \begin{center}
        \includegraphics[width=1\linewidth]{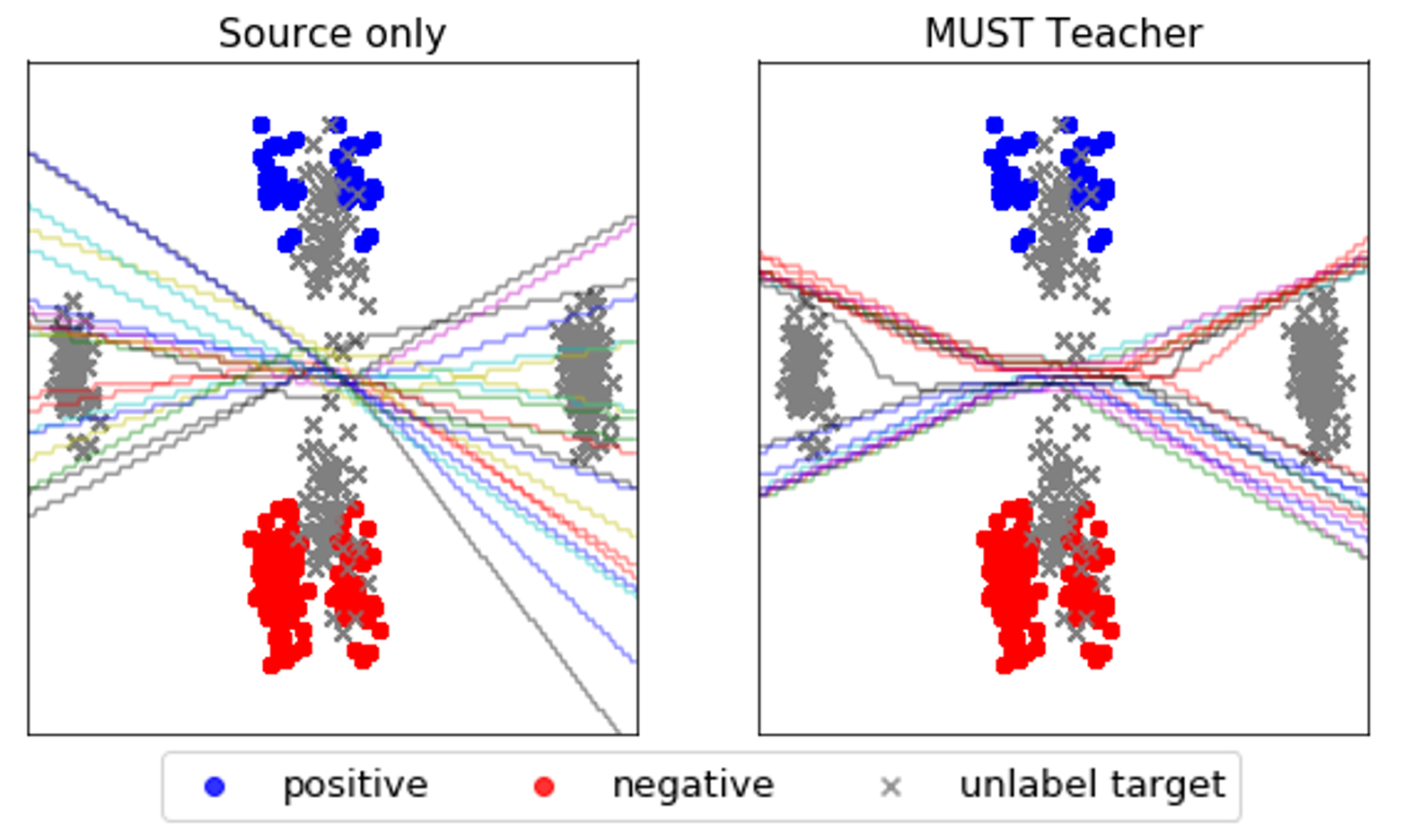}
    \end{center}
    \caption{Decision boundaries of source-only models and MUST for 20 initializations, each with a different color. Blue: positive source samples. Red: negative source samples. Gray: unlabeled target samples. The source-only model classifies perfectly the source data but ignores the target data. The MUST teacher learns to classify the source data and at the same time avoids dense areas of the target distribution.}
    \label{fig:avoid-clusters}
\end{figure}

To assess if this effect occurs in realistic datasets, we measured the sample density as a function of the distance from the decision boundary. Specifically, we perturbed each sample by epsilon in the  direction of an adversarial perturbation and counted the how many samples changed their labels. Figure \ref{fig:adv} shows that for the \textit{source-only} model, many more samples are close to the decision boundary, compared with the MUST model. This shows that MUST decision boundaries have larger margin for target samples. 

\begin{figure}[!h]
    \begin{center}
        \includegraphics[width=0.85\linewidth]{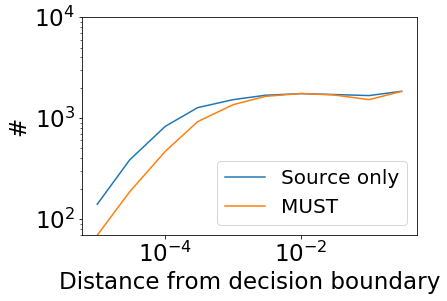}
    \end{center}
    \caption{The number of target samples within a specific distance from decision boundary. For each target sample, we move an epsilon step in direction of adversarial perturbation and count the number of samples that change their labels.}
    \label{fig:adv}
\end{figure}

\subsection{Ablation}
We quantify the contribution of various components of MUST. (1) Only-BN: We trained a teacher network without the student loss regularization using per-domain batch norm layers. (2) Teacher: use the teacher as the final classifier. (3) MUST: use the student as the final classifier. 
Table \ref{tab:amazon-ablation} summarizes the results for the sentiment analysis experiment, Table \ref{tab:digits-ablation} for digit recognition and Table \ref{tab:domainnet-ablation} for DomainNet.   

\begin{table}[h!]
    \centering
    \scalebox{0.95}{
    \begin{tabular}{ |l|c|c|c|c||c| } 
        Method & Books & DVD & Elect & Kitn & Avg.\\ 
        \hline
        Only-BN & 75.5 & 78.2 & 82.3 & 84.2 & 80.1 \\ 
        Teacher & 76.0 & 79.1 & 82.6 & 84.7 & 80.6\\ 
        MUST & \textbf{80.5} & \textbf{81.9} & \textbf{86.3} & \textbf{87.9} & \textbf{84.2}\\ 
        \hline
    \end{tabular}
    }
    \caption{MUST ablation on Amazon product reviews}
    \label{tab:amazon-ablation}
\end{table}

\begin{table}[h!]
    \centering
    
    \scalebox{0.95}{
    \begin{tabular}{ |l|c|c|c|c||c| } 
        Method & MNIST & M-M & SVHN & SYN & Avg. \\ 
        \hline
        Only-BN & 98.2 & 72.6 & 76.3 & 94.5 & 85.4 \\ 
        Teacher & \textbf{98.9} & 83.2 & \textbf{86.0} & 95.9 & 91.0 \\ 
        MUST & \textbf{98.9} & \textbf{83.8} & \textbf{86.0} & \textbf{96.1} & \textbf{91.2} \\ 
        \hline
    \end{tabular}
    }
    \caption{MUST ablation on digit recognition}
    \label{tab:digits-ablation}
\end{table}

\begin{table}[h!]
    \centering
    \setlength{\tabcolsep}{2pt}
    {\small 
    \sc
    \scalebox{0.99}{
    \begin{tabular}{ |l|c|c|c|c|c|c||c| } 
        & Clip-  & Info-  & Paint  & Quick-  & Real & Sketch  & Avg.  \\
        Models & art & graph & & draw &  &  &   \\\hline
        Only-BN & 47.9 & 13.5 & 39.9 & 9.3 & 55.9 & 36.6 & 33.8 \\ 
        Teacher & 58.9 & 8.1 & 47.9 & 3.3 & \textbf{65.2} & 49.3 & 38.8 \\
        MUST & \textbf{60.8} & \textbf{20.5} & \textbf{48.2} & \textbf{12.2} & 65.1 & \textbf{49.8} & \textbf{42.8}\\ 
        \hline
    \end{tabular}
    }}
    \caption{MUST ablation on DomainNet}
    \label{tab:domainnet-ablation}
\end{table}

In most cases, the teacher improves the Only-BN score, demonstrating the influence of the student feedback to the teacher. In addition, the student is consistently better than the teacher, indicating the reduction of negative transfer. This effect is particularly noticeable on Quickdraw where the negative transfer for all MSDA methods were shown in section \ref{sec:domainnet}.

\section{Conclusion}
This paper describes a new training procedure for MSDA. We empirically show that MUST reduced the error on target data in three real life datasets, achieving state-of-the-art results on digit recognition, Amazon product review and DomainNet datasets.

\section*{Acknowledgement}
Research was supported by the Israel Science Foundation (ISF grant 737/2018)  and by the Israeli innovation authority, the AVATAR consortium.

\small
\bibliography{cite.bib}

\end{document}